\documentclass[12pt]{alt2022} 

\setcitestyle{numbers,open={[},close={]}} 

\title[Universal Online Learning with Unbounded Losses: Memory Is All You Need]{Universal Online Learning with Unbounded Losses:\\ Memory Is All You Need}
\usepackage{times}
\usepackage{soul}
\usepackage[mathscr]{eucal}

\altauthor{%
 \Name{Mo\"ise Blanchard} \Email{moiseb@mit.edu}\\
 \addr Massachusetts Institute of Technology
 \AND
 \Name{Romain Cosson} \Email{cosson@mit.edu}\\
 \addr  Massachusetts Institute of Technology
 \AND
 \Name{Steve Hanneke} \Email{steve.hanneke@gmail.com}\\
 \addr Purdue University
}

\newcommand{\comment}[1]{}

\newcommand{\suol}{\text{SUOL}}
\newcommand{\sual}{\text{SUAL}}
\newcommand{\suil}{\text{SUIL}}
\newcommand{\fs}{\text{FS}}
\newcommand{\fmv}{{\text{FMV}}}

\newcommand{\bq}{\boldsymbol q}

\newcommand{\rvM}{\text{M}}

\newcommand{\Acal}{\mathcal{A}}
\newcommand{\Bcal}{\mathcal{B}}
\newcommand{\Ccal}{\mathcal{C}}

\newcommand{\Ecal}{\mathcal{E}}
\newcommand{\Fcal}{\mathcal{F}}

\newcommand{\Lcal}{\mathcal{L}}

\newcommand{\Pcal}{\mathcal{P}}

\newcommand{\Xcal}{\mathcal{X}}
\newcommand{\Ycal}{\mathcal{Y}}

\newcommand{\Ebb}{\mathbb{E}}
\newcommand{\Nbb}{\mathbb{N}}
\newcommand{\Pbb}{\mathbb{P}}
\newcommand{\Qbb}{\mathbb{Q}}
\newcommand{\Rbb}{\mathbb{R}}

\newcommand{\Xbb}{\mathbb{X}}
\newcommand{\Ybb}{\mathbb{Y}}

\newcommand{\fb}{\bar{f}}

\usepackage{bbm}
\newcommand{\one}{\mathbbm{1}}
\newcommand{\1}{\mathbbm{1}}

\definecolor{dark_red}{rgb}{0.2,0,0}

\newcommand{\N}{\mathbb{N}}

\newcommand{\mb}[1]{\ensuremath{\boldsymbol{#1}}}

\newcommand{\metric}{\rho}

\usepackage{graphicx}

\renewenvironment{proof}[1][]{\par\noindent{\bf Proof #1\ }}{\hfill\BlackBox\\[2mm]}

\begin{document}

\maketitle

\begin{abstract}%
We resolve an open problem of \cite{hanneke2021learning} on the subject of universally consistent online learning with non-i.i.d.\ processes and unbounded losses.  The notion of an \emph{optimistically universal learning rule} was defined by Hanneke \cite{hanneke2021learning} in an effort to study learning theory under minimal assumptions. A given learning rule is said to be optimistically universal if it achieves a low long-run average loss whenever the data generating process makes this goal achievable by some learning rule.  \cite{hanneke2021learning} posed as an open problem whether, for every unbounded loss, the family of processes admitting universal learning are precisely those having a finite number of distinct values almost surely.  In this paper, we completely resolve this problem, showing that this is indeed the case. As a consequence, this also offers a dramatically simpler formulation of an optimistically universal learning rule for any unbounded loss: namely, the simple \emph{memorization} rule already suffices.  Our proof relies on constructing random measurable partitions of the instance space and could be of independent interest for solving other open questions \cite{hanneke2021open}. We extend the results to the non-realizable setting thereby providing an optimistically universal Bayes consistent learning rule.
\end{abstract}


\section{Introduction}
\paragraph{Online learning.} One of the main classical questions in statistical learning is \emph{learnability}: whether it is possible to have low prediction loss on a prediction task given observations. In this paper, we study this question in the context of \emph{online learning}.  In this setting, there is a possibly random sequence of points $\Xbb = (X_t)_{t \in \N}$ from a space $\Xcal$ of inputs and target values $\Ybb = (Y_t)_{t \in \N}$ from a space of outputs $\Ycal$. These two sequences are stochastic processes in general. Learning occurs sequentially, where at each time $t$ the learner has observed $(X_s)_{s \leq t-1}$ and $(Y_s)_{s \leq t-1}$ and $X_t$, and makes a prediction $\hat{Y}_t \in \Ycal$ for the value of $Y_t$.  Thus, we may consider $\hat{Y}_t = f_t((X_1,Y_1),\ldots,(X_{t-1},Y_{t-1}),X_t)$ for some function $f_t$ (possibly randomized).  We will assume that there is a measurable \emph{target} function $f^*: \Xcal \to \Ycal$ such that $Y_t = f^*(X_t)$: that is, the target values $Y_t$ are deterministic in the input.  However, we place no restrictions on this function.  Given a loss function $\ell$, the performance of the algorithm is quantified by the value $\ell(\hat Y_t, Y_t)$. In particular, we will say that an algorithm has learnt the prediction task if it guarantees a low long-run average loss, i.e. $\frac{1}{n}\sum_{t=1}^{n}\ell(\hat{Y}_t,Y_t)\xrightarrow[n\rightarrow \infty]{} 0~~ (a.s).$

\paragraph{Universal online learning.} In general, it is not possible to obtain such guarantees for all sequences $\Xbb$ and functions $f^*$. It is therefore necessary to constrain the set of considered pairs $(\Xbb,f^*)$.  There is a rich literature studying online learning with unrestricted sequences $\Xbb$ but with restrictions on the function $f^*$ 
\cite{littlestone1988learning,ben2009agnostic}, or with a mix of restrictions on $\Xbb$ and $f^*$ \cite{haussler1994predicting,ryabko2006pattern,urner2013probabilistic,bousquet2021theory}.  Here, we focus on \emph{universal learning}, which imposes no assumptions on the set of target functions $f^*$, but restricts the input sequences $\Xbb$.  Thus, we are interested in algorithms which are \emph{universally consistent} under a given family of stochastic processes $\Xbb$. For instance, a classic result in this line of work states that in the Euclidian space $\Xcal=\Rbb^n$, even the inductive \emph{nearest neighbor} predictor for classification has vanishing error rate for every i.i.d.\ input sequence $\Xbb$ and every target $f^*$ \cite{cover1967nearest,stone1977consistent,devroye2013probabilistic}, and thus can be shown to have vanishing long-run average loss in the online setting for this case as well.

In this work, we primarily focus on \emph{unbounded} loss functions: i.e., the case $\sup_{y,y' \in \Ycal} \ell(y,y') = \infty$.  
In the context of i.i.d.\ processes, or various extensions thereof, there have been many works that consider unbounded losses, but with additional restrictions on the $Y_t$ sequence (which effectively allow them to reduce back to the bounded case, in a certain sense).  For instance, in the i.i.d.\ setting, \cite{gyorfi:02} presents a variety of results on universal consistency for regression when $Y_t$ has finite \emph{variance}.  Similarly, \cite{gyorfi:07,gyorfi:12,biau:10} develop several results on universal consistency for regression under \emph{stationary ergodic} processes, under the restriction that $Y_t$ has finite fourth moments.
Extending these results to admit general families of processes $(\Xbb,\Ybb)$ is certainly an interesting and worthy direction of research.
However, in the present work we focus on the case of \emph{unrestricted} $Y_t$ sequences, aside from the aforementioned assumption that $Y_t = f^*(X_t)$.
In particular, in the regression setting, there are no variance restrictions on $Y_t$.  
As a consequence, the results we arrive at necessarily impose stricter requirements on the process $\Xbb$ of inputs.
Indeed, one of the main contributions of the present work is illuminating just how restrictive the requirement on $\Xbb$ must necessarily be for universal consistency to be possible with unbounded losses.

Most of the literature on universal consistency has relied on conventional assumptions on stochastic processes, imported from the probability theory literature, such as the i.i.d.\ assumption or relaxations to stationnary ergodic \cite{morvai1996nonparametric,gyorfi1999simple,gyofi2002strategies} or satisfying a law of large numbers \cite{morvai1999regression,steinwart2009learning}. In contrast, in an effort to formulate a theory of learning under minimal assumptions, \cite{hanneke2021learning} recently introduced a framework that uses provably-minimal assumptions on $\Xbb$. Specifically, the sole assumption on the input sequence is that there exists some algorithm which achieves universal consistency. This is known as the so-called optimist's decision theory: in order to achieve a given objective -- in our case universal consistency -- the minimal assumption is that this objective is at least possible.  In that sense, the optimist's assumption is minimal, as it is a necessary assumption to prove any positive results on universal consistency. We are then particularly interested in designing online learning rules with guarantees that hold without further assumptions, which are referred to as \emph{optimistically universal} learning rules. Such learning rules are consistent whenever learning is possible.  Equivalently, if an optimistically universal learner fails for a specific instance, any other strategy would fail to be universally consistent under that process $\Xbb$ as well.

In the case of \emph{unbounded} losses $\ell$, the existence of an optimistically universal online learning rule was settled by \cite{hanneke2021learning}.

This work also expresses a condition (condition $\fmv$ below) which characterizes the family of processes $\Xbb$ that admit the existence of universally consistent online learning rules for any (and all) unbounded losses.
However, the definition of the optimistically universal learning rule given in that work, the proof that it satisfies this property, 
and also the proofs establishing that the proposed condition indeed characterizes the relevant family of processes, are actually quite complex.
For instance, the learning rule involves identifying a function contained in a certain countable function class $\tilde{\Fcal}$, satisfying constraints on its losses relative to various other values, 
and the proof proceeds via arguing that there exists a choice of $\tilde{\Fcal}$ that is \emph{dense} in the set of all measurable functions, in a sense relevant to 
learning under every $\Xbb$ satisfying the condition.
However, \cite{hanneke2021learning} also poses an interesting open problem (Open Problem 4 there) regarding a potential dramatic simplification of this theory.  The essential question is the following:

\paragraph{Open Problem \cite{hanneke2021learning}:} \textit{For unbounded losses, is it true that there exist universally consistent online learning rules under $\Xbb$ if and only if $\Xbb$ almost surely has a finite number of distinct elements?}

\paragraph{Summary of contributions.} 
We completely resolve the above question, proving that the original condition proposed by \cite{hanneke2021learning} (condition $\fmv$ below) 
is equivalent to the condition that $\Xbb$ almost surely contains only a finite number of distinct elements.  
Therefore, for unbounded losses, there exist universally consistent online learning rules under $\Xbb$ if and only if $\Xbb$ almost surely contains only a finite number of distinct elements.
This result has immediate implications for drastically simplifying the theory of universally consistent learning with unbounded losses. 
Rather than the complicated learning strategy proposed by \cite{hanneke2021learning}, it suffices to use the simple \emph{memorization} algorithm, which simply remembers all past data points $(X_s,Y_s)$, $s < t$, and if the new $X_t$ satisfies $X_t = X_s$ for some $s < t$, it predicts $Y_s$.  If $\Xbb$ has only a finite number of distinct elements, then clearly this strategy has only finitely many non-zero losses, and hence would be universally consistent.  Indeed, since this is the case for all such $\Xbb$, for any unbounded loss this algorithm is also \emph{optimistically} universal.

{\vskip 2mm}\noindent We refer to the above simple condition as the \emph{finite support} ($\fs$) condition:

{\vskip 3mm}\noindent{\bf Condition \fs}~~ {\it Define ${\normalfont \fs}$  as the set of all stochastic processes $\Xbb$ such that}
\begin{equation*}
    |\{x\in \mathcal X: \Xbb \cap \{x\}\neq \emptyset\}| < \infty \quad (a.s.).
\end{equation*}

\noindent
As in \cite{hanneke2021learning}, we let $\suol$ denote the family of all processes $\Xbb$ under which 
there exist universally consistent online learning rules (defined more formally below).
Our main result can then be stated as follows.

\begin{theorem}
\label{thm:caracterization_suol}
For $\Xcal$ any separable metric space and $\ell$ any unbounded loss, $\normalfont\suol=\fs$.
\end{theorem}

\noindent
Although the existence of an optimistically universal learning algorithm was shown in \cite{hanneke2021learning} for unbounded losses, the above new characterization by $\fs$ drastically simplifies the definition of such a learning rule. As mentioned, as a consequence, we will prove that the simple ``memorization rule" (described above) is optimistically universal.

\begin{theorem}
\label{thm:memorization_opt_universal}
If $\Xcal$ is a separable metric space and the loss is unbounded, the memorization rule is an optimistically universal online learning rule. 
\end{theorem}

\paragraph{Outline of paper.} The rest of this paper is structured as follows. We introduce the formal setup and preliminaries in Section \ref{section:background_preliminaries}. We prove the main theorem on universal learning in Section \ref{section:main_result}, starting with the case $\Xcal = [0,1]$ and generalizing it to hold for any separable metric space in Section \ref{sec:separable}. We discuss consequences of the result for \emph{inductive} and \emph{self-adaptive} learning in Section \ref{section:consequences}. Finally, in Section \ref{sec:bayesian}, we consider a noisy setting and prove the existence of an optimistically universal Bayes consistent learning rule when the loss is unbounded. Remaining open problems on optimistically universal online learning will be recalled in the conclusion in Section \ref{sec: conclusion}.

\section{Background and Preliminaries}
\label{section:background_preliminaries}
\subsection{Formal setup}
\paragraph{Input space.} We consider the general setup where $(\Xcal, \metric)$ is a separable metric space. We define the set $\Bcal$ of measurable subsets of $\Xcal$ to be the $\sigma$-algebra generated by the topology induced by the metric $\metric$. For more details on this setup, we refer to \cite{parthasarathy2005probability}.

\paragraph{Value space and loss function.} We fix the value space $\Ycal$, and an associated loss function $\ell:\Ycal^2\rightarrow [0,\infty)$ that satisfies 
a relaxed triangle inequality $\forall y_1, y_2, y_3\in \Ycal^3: \ell(y_1,y_3)\leq c_{\ell} (\ell(y_2,y_1)+\ell(y_2,y_3))$, where $c_{\ell}$ is a finite constant, as well as that 

$\forall y\in \Ycal, \ell(y,y)=0$.  
For instance, the squared loss in regression satisfies this with $c_{\ell} = 2$.

While there remain important open questions for universal online learning with \emph{bounded} loss functions \cite{hanneke2021open}, in the present work we focus on \emph{unbounded} loss functions: that is, we assume $\sup_{y_1, y_2\in\Ycal}\ell(y_1,y_2)=\infty$.

\paragraph{Data generation process.} The input data is generated by an infinite-horizon stochastic process $\Xbb = \{X_t\}_{t=1}^{\infty}$, taking its values in $\Xcal$. We impose no constraint on the nature of this process. We assume that the output data $\Ybb = \{Y_t\}_{t=1}^{\infty}$ is generated from $\Xbb$ by a deterministic measurable function $f^*:\Xcal\rightarrow\Ycal$. In other words, we have $\{Y_t\}_{t=1}^{\infty} = \{f^*(X_t)\}_{t=1}^{\infty}$. Note that this framework substantially differs from other setups where the response function is noisy. When we look at a bounded time horizon $t\geq 1$, we will use the following notation: $\Xbb_{\leq t} = \{X_1,...,X_t\}$ and $\Xbb_{< t} = \{X_1,...,X_{t-1}\} $. We will also abuse notations and allow ourselves to use $\Xbb$ to denote the set of all the values taken by the stochastic process. For instance we will write $\#\Xbb = \#\{x\in \Xcal: \{x\}\cap \Xbb \neq \emptyset\}\in \Nbb\cup{\{+\infty\}}$ the number of different values taken by $\Xbb$, we will also write $\#\Xbb_{\leq t}$ the number of different values taken by $\Xbb$ before time $t$.

\paragraph{Online learning rule.} An online learning rule is formally defined as a sequence $\{f_t\}_{t=1}^{\infty}$ of measurable functions $f_t:\Xcal^{t-1}\times \Ycal^{t-1}\times\Xcal \rightarrow \Ycal$. Given $t-1$ training examples of the form $(X_i, f^*(X_i))\in \Xcal\times \Ycal$ and one new sample $X_t$, the learning rule $f_t$ makes prediction $f_t(\Xbb_{< t}, \Ybb_{< t}, X_t)$ for $f^*(X_t)$. As an important example, the \emph{memorization learning rule} $\{f_t\}_{t=1}^{\infty}$ is defined as follows: 
\begin{equation*}
    f_t((x_i)_{i<t},(y_i)_{i<t},x_t) = \begin{cases}
    y_i &\text{if } x_t = x_i,\\
    y_0 &\text{if } x_t \not\in \{x_i\}_{i<t}, 
    \end{cases}
\end{equation*}
where $y_0\in \Ycal$ is some arbitrary default response. 

\paragraph{Learning task.} The goal we pursue is to minimize the online loss defined as,
\begin{equation*}
    \Lcal_{\Xbb}(f_{.},f^*,T) = \frac{1}{T}\sum_{t=1}^{T}\ell(f_t(\Xbb_{< t},\Ybb_{< t}, X_{t}), f^*(X_t)).
\end{equation*}
Specifically, letting 
\begin{equation*}
    \Lcal_{\Xbb}(f_{.},f^*)=\limsup_{T \to \infty}\Lcal_{\Xbb}(f_{.},f^*; T),
\end{equation*}
we would like to guarantee $\Lcal_{\Xbb}(f_{.},f^*)=0 ~~(a.s.)$.  
This motivates the following definition of universal consistency.

\paragraph{Universal consistency and optimistically universal learning rules.} We recall here the full statement of the definitions introduced by \cite{hanneke2021learning}. We say an online learning rule $\{f_t\}_{t=1}^{\infty}$ is \emph{strongly universally consistent} under $\Xbb$ if for every measurable $f^*:\Xcal\rightarrow \Ycal$, we have that $\Lcal_{\Xbb}(f_{.},f^*) = 0~~ (a.s.).$ We say the process $\Xbb$ admits \emph{strong universal online learning} if there exists an online learning rule $\{f_t\}_{t=1}^{\infty}$ that is strongly universally consistent under $\Xbb$. We denote by {\normalfont $\suol$} the set of all processes $\Xbb$ that admit strong universal online learning. An online learning rule $\{f_t\}_{t=1}^{\infty}$ is said to be \emph{optimistically universal} if it is universally consistent under every $\Xbb$ in {\normalfont $\suol$}.\\

\noindent In this line of work, two of the main interests are:  
(1) providing concise conditions characterizing the family $\suol$ in terms of properties of the stochastic process $\Xbb$, 
and (2) identifying particular learning rules that are optimistically universal: that is, learning rules that are universally consistent under every $\Xbb$ in $\suol$. 
Establishing the existence of optimistically universal online learning rules 
remains an open problem for all \emph{bounded} loss functions. 
In the case of \emph{unbounded} losses this question has been settled. 
Specifically, \cite{hanneke2021learning} shows that, for any unbounded loss, there exists optimistically universal online learning rules.
Moreover, \cite{hanneke2021learning} also expresses a condition which characterizes the family $\suol$.  
The condition requires that, for every countable measurable partition of $\Xcal$, the process $\Xbb$ visits a finite number of cells almost surely. This will be referred to as the ``finite measurable visits" ($\fmv$) condition:\\

\noindent{\bf Condition \fmv ~\cite{hanneke2021learning}}~~ {\it Define the set ${\normalfont\fmv}$ as the set of all processes $\Xbb$ satisfying the condition that, for every disjoint sequence $\{A_k\}_{k=1}^{\infty}$ in $\Bcal$ with $\cup_{k=1}^{\infty}A_k=\Xcal$ (i.e., every countable measurable partition),}
\begin{equation*}
    \#\{k\in \mathbb N: A_k\cap \mathbb{X} \neq \emptyset \}  < \infty\quad (a.s).
\end{equation*}

{\vskip 2mm}It is worth noting, however, that the specification and analysis of the optimistically universal learning rule in \cite{hanneke2021learning}, and the proof that Condition $\fmv$ indeed characterizes $\suol$, are all quite complicated.  For instance, the algorithm and its analysis directly rely on a construction of a countable dense subset of the set of measurable functions, under a metric appropriate to learning with unbounded losses.  The learning algorithm then solves a sequence of constraint satisfaction problems specified in terms of this countable dense set of functions, to select one such function as its predictor.  It is therefore desirable to \emph{simplify} the theory, not only for practical reasons, but also to help us to intuitively understand the varieties of processes that admit universally consistent learners, and to clarify what kinds of learning rules can be optimistically universal.  Toward this end, \cite{hanneke2021learning} poses an important open question regarding a potential simplification: is $\suol$ characterized by Condition $\fs$?
\comment{
\noindent\textbf{Open Problem \cite{hanneke2021learning}:} \textit{For unbounded losses, is $\suol = \fs$?}\\
}
The main contribution of the present work is showing that indeed this is \emph{true} (Theorem~\ref{thm:caracterization_suol}).
This fact allows us to dramatically simplify the entire theory of universally consistent online learning with unbounded losses.
Several simplifications are immediate from this: 
\begin{enumerate}
\item This provides a new, stronger characterization of $\suol$. 
\item The proof establishing $\fs = \suol$ is significantly simpler than the original proof that $\fmv = \suol$.
\item The equivalence $\fs = \suol$ immediately implies that the simple \emph{memorization} rule is optimistically universal.
This contrasts with the complicated construction used in the optimistically universal learner of \cite{hanneke2021learning}, 
which solves a sequence of constraint satisfaction problems in terms of a countable dense set of measurable functions.
\end{enumerate}

{\vskip 1mm}\paragraph{A remark on dependency in the problem setup.} At first glance, one might think that the class $\suol$ and its characterization should somehow depend on the specific setup given by $(\Xcal,\metric)$, $\Ycal$ and $\ell$. However, as \cite{hanneke2021learning} showed (and as is implied by our Theorem~\ref{thm:caracterization_suol}), this dependency is very mild. In particular, the existence of an optimistically universal learning rule does not depend on the choice of $(\Ycal,\ell)$ as long as the loss is unbounded: $\sup_{y,y'\in \Ycal}\ell(y,y') = \infty$.  Moreover, our results will hold for any separable metric space $(\Xcal,\metric)$.  

\paragraph{Outline of the proof.} The essential strategy of the proof relies on the fact that $\fs \subset \suol \subset \fmv$. 
The left inclusion is rather obvious.  Indeed, if $\Xbb$ contains a finite number of distinct values (a.s.), even the simple \emph{memorization} learning rule is universally consistent.
For the sake of thoroughness, we include a brief proof of this observation in Section~\ref{sec:fs_sufficient}.  
The second inclusion, that $\suol \subset \fmv$, was shown by \cite{hanneke2021learning} as part of the proof that $\suol = \fmv$. 
We state this result formally in Section~\ref{subsection:necessary_condition}, and for the sake of being self-contained we include its proof in Appendix~\ref{appendix:proof_fmv_necessary}.
Given these inclusions $\fs \subset \suol \subset \fmv$, what remains is establishing that $\fmv = \fs$.  
Establishing this equivalence is the main technical contribution of this work (Theorem~\ref{thm:fs=fmv}). Its proof relies on constructing random measurable partitions of the space $\Xcal$.
We now turn to discussing the details of each of these components.

\subsection{Sufficient condition for $\suol$}
\label{sec:fs_sufficient}

We begin with the easiest of the claimed inclusions: namely, $\fs \subset \suol$.
Recall that condition $\fs$ corresponds to having a finite number of values almost surely, i.e. $\#\Xbb< \infty \quad (a.s.)$. 
While it may be rather obvious that all such processes admit a strong universal learning rule (i.e., they belong to $\suol$), for the sake of thoroughness we present a simple proof of this fact.

\begin{proposition}
\label{prop:fs_sufficient}
${\normalfont\fs}\subset{\normalfont\suol}$.  ~In particular, the memorization rule is universally consistent under every $\Xbb \in {\normalfont\fs}$.
\end{proposition}
\begin{proof}
We will show that the memorization learning rule is universally consistent under every $\Xbb$ that takes a finite number of values almost surely. We can formally prove this result as follows. Let $\Xbb\in \fs$ be a given stochastic process and let $\{f_t\}_{t=1}^{\infty}$ be the memorization learning rule defined earlier. Observe that, for any measurable target function $f^* : \Xcal \to \Ycal$, the (random) quantity $\rvM = \max_{t \in \N}\ell(y_0,f^*(X_{t}))$ is always finite (a.s.), as it is a maximum over a finite set: $\rvM<\infty~~ (a.s.)$. Now observe that the memorization rule makes at most $\#\Xbb$ errors, each of value at most $\rvM$. Therefore, $0\leq \hat{\Lcal}_{\Xbb}(f_{.},f^*,T) \leq \frac{1}{T}\rvM \cdot \#\Xbb \xrightarrow[T\rightarrow\infty]{} 0\quad (a.s.)$.
\comment{The proof above is very efficient, but hides the complexity of the reasoning on stochastic processes. In order to prepare for the next section, we present a more elementary proof. For $k\in \Nbb$, we define $C_k\in \Nbb$ and $M_k$ such that:
\begin{equation*}
    \Pbb(|\Xbb| \leq C_k)\geq 1- \frac{1}{2^k}\quad \text{and} \quad \Pbb(\rvM\leq M_k)\geq 1-\frac{1}{2^k}.
\end{equation*}
Which gives that for any $k\in \Nbb$
\begin{equation*}
    \Pbb\left(\lim_{n\rightarrow \infty}\hat{\Lcal}_{\Xbb}(f_{.},f^*,n) = 0\right)\geq 1-\frac{1}{2^{k-1}}.
\end{equation*}
In other words,
\begin{equation*}
    \lim_{n\rightarrow \infty}\hat{\Lcal}_{\Xbb}(f_{.},f^*,n) = 0 \quad (a.s.).
\end{equation*}
}
\end{proof}
An important remark is that this condition is not \emph{testable}: there does not exists a consistent hypothesis test for $\fs$. In other terms it is not possible to decide from the stream of input data $\Xbb$ whether the process satisfies $\fs$ or not. Formally, a \emph{hypothesis test} refers to a sequence of possibly random decision functions $\hat t_n:\Xcal \to\{0,1\}$. We then say that a test is \emph{consistent} for a class of processes $\Ccal$ if for any process $\Xbb$, $\hat t_n(\Xbb_{\leq n}) \to \mb 1_{\Xbb\in\Ccal}$ in probability. 
\begin{proposition}
\label{prop:no_test_fs}
If $\Xcal$ is infinite, there is no consistent hypothesis test for condition $\normalfont\fs$.
\end{proposition}

\begin{proof}
This is a consequence from Theorem \ref{thm:caracterization_suol} and the fact that there is no consistent hypothesis test for $\suol$ when $\Xcal$ is infinite, shown in \cite{hanneke2021learning} (Theorem 59). For completeness, we provide a direct and simplified proof in Appendix \ref{appendix:no_consistent_test_fs}.
\end{proof}
This justifies the terminology ``optimistic" in \textit{optimistically universal learning rule} in the sense that belonging to the set of sequences for which universal learning is achievable, which we will prove is equal to $\fs$, is a non-testable assumption.

\subsection{Necessary condition for $\suol$}
\label{subsection:necessary_condition}
We now recall a necessary condition for a stochastic process $\Xbb$ to admit strong universal online learning. It was shown that condition $\fmv$, which requires that $\Xbb$ will only visit a finite number of zones for all given countable measurable partitions of $\Xcal$, is necessary for universal learning.

\begin{theorem}[\cite{hanneke2021learning}]
\label{thm:fmv_necessary}
$\normalfont \suol \subset \fmv$.
\end{theorem}

\noindent
We present the proof idea here, while for the purpose of being self-contained, a brief version of the full proof is included in Appendix \ref{appendix:proof_fmv_necessary}. Suppose the process $\Xbb$ does not satisfy $\fmv$, then there exists a measurable partition $\{A_k\}_{k=1}^{\infty}$ such that with positive probability, $\Xbb$ visits an infinite number of the $A_k$. Now define a function $f^*$ that is randomly piece-wise constant, i.e. constant on each $A_k$ and taking random value $f^*(x\in A_k) \in \{y_{k,0},y_{k,1}\} \in \Ycal$ where $\ell(y_{k,0},y_{k,1})$ grows large at a sufficiently fast rate. Any online learning rule will fail to predict $Y_t = f^*(X_t)$ an infinite number of times, inducing non-vanishing loss. The reason is that the learning rule cannot leverage any of the training data $(\Xbb_{<t},\Ybb_{<t})$ when $A_k$ is visited for the first time at time $t$. 

\subsection{Open Problem 4}
In the previous sections, we recalled the inclusions $\fs\subset\suol\subset\fmv$. This allows for a concise expression of the conjecture formulated by \cite{hanneke2021learning} and referred to as ``Open Problem 4" (which, in light of the result of \cite{hanneke2021learning} that $\suol = \fmv$, is equivalent to the formulation of the problem stated earlier in Section~\ref{section:background_preliminaries}). \\

\noindent{\bf Open Problem 4 (\citep{hanneke2021learning})} {\it Is it true that $\normalfont \fmv = \fs$?}\\

\noindent
We will prove that this equality holds for all separable metric spaces $\Xcal$. This in turn implies that $\suol = \fs$ and therefore ensures that ``memorization", which we already saw is universally consistent for all processes in $\fs$ (Proposition~\ref{prop:fs_sufficient}), is an optimistically universal learning rule (thus establishing Theorem~\ref{thm:memorization_opt_universal}). The solution to the open problem will be detailed in Section \ref{section:main_result} and generalised to all separable metric spaces in Section \ref{sec:separable}. We conclude this section by giving some additional inspiration for the proofs that will follow. 

\noindent
\paragraph{Remarks on Open Problem 4.} In words, the question asked by Open Problem 4 is whether the set of countable measurable partitions is sufficiently large to separate all stochastic processes that take an infinite number of values.\\

\noindent
It was already observed by \cite{hanneke2021learning} that when $\Xcal$ is countable or when $\Xbb$ is deterministic, $\fs$ and $\fmv$ are equal. However, both these setups come with a natural partition: if $\Xcal$ is countable, $\{\{x\} : x\in \Xcal\}$ becomes a countable measurable partition of $\Xcal$, and when 
$\Xbb$ is deterministic $\{ \{x\} : \{x\} \cap \Xbb \neq \emptyset \} \cup \{ \Xcal\setminus \Xbb \}$ will also isolate all the different values taken by $\Xbb$.\\

\noindent
In the uncountable case, for instance when $\Xcal= \Rbb$, we aim to define a partition $\{A_k\}_{k=1}^{\infty}$ that scatters the space. We want to minimize the chance that two values taken by the process, say $X_t\neq X_{t'}$, fall in the same $A_k$. A classical and tempting way to build such a partition would be using the axiom of choice \cite{vitali1905sul,stern1985probleme}. Define the equivalence relation $x\sim_{\Qbb} y \iff x-y\in \Qbb$, where $\Qbb$ can be enumerated $\Qbb = \{q_1,q_2,...\}$. Now for each of the equivalence classes of the form $\{x\}+\Qbb$, \textit{choose} one representer. Denote by $A$ the set of all representers and observe that $\{A+\{q\}\}_{q\in \Qbb}$ makes a countable partition of $\Rbb$. Note that two different values of $\Xbb$, say $X_t\neq X_{t'}$, fall in the same equivalence class only if $X_t-X_{t'}$ was chosen as a representer. This event could be made very rare if we were to shift all representers by a uniform random variable, or to choose the representer at random within their class of equivalence. The reason why this does not prove the result is that the corresponding partition $\{A_k\}_{k=1}^{\infty}$ is not measurable.\\

\noindent
Another idea to create such a random partition $\{A_k\}_{k=1}^{\infty}$ would be to assign each $x\in \Rbb$ to a set $A_{k(x)}$ where the index $k(x)\in \Nbb$ is chosen independently at random following an exponential law $\Ecal(\frac{1}{2})$: $\Pbb(k(x)=k)=\frac{1}{2^k}$. The indices $\{k(X_1),k(X_2),...\}$ to which the elements of the sequence $\Xbb = \{X_1, X_2, ...\}$ are assigned, are almost surely unbounded when $\#\Xbb = +\infty$, disproving condition $\fmv$. We will refer to this construction as the partition $\Pcal$ as it will later be a useful inspiration. Unfortunately, as such, $\mathcal{P}$ not define a proper partition because the sets $A_k$ are not measurable in general. To solve this issue, instead of defining point-wise random sets, we will use countable union of small intervals. Depending on the scale of the process $\Xbb$, these sets will give same behaviour as the parts of $\Pcal$. We will make this idea more precise in the following paragraph.\\ 

\noindent
We first recall a construction of dense open sets of $\Rbb$ with measure at most $\epsilon>0$. Following  a classical argument, one can consider the union of open intervals $\cup_{i\geq 1}\left(q_i-\frac{\epsilon}{2^i},q_i+\frac{\epsilon}{2^i}\right)$, where $\{q_1,q_2,...\}$ are i.i.d. sampled from some probability density of full support. If we denote the remainders $R_k = \cup_{i\geq k}\left(q_i-\frac{\epsilon}{2^i},q_i+\frac{\epsilon}{2^i}\right)$ and consider the partition $\{A_k\}_{k=0}^{\infty}$ defined by $A_k = R_{k}\setminus R_{k+1}$ where $R_0=\Rbb$, one could hope that any sequence $\Xbb$ taking infinite values will visit an infinite number of the $A_k$. In fact, this is true if the convergence rate of $\Xbb$ is not too fast but not in the general case. We will therefore use a decay rate adapted to the process $\Xbb$ through a parameter $\delta_k$ defined as follows,
\begin{equation*}
    \delta_k = \min\!\left\{|x-y| \Big| x,y\in \Xbb_{\leq N}, \#\Xbb_{\leq N}\geq 2^{2k+2}\right\}.
\end{equation*}
A key intuition is that the first $k$ distinct points visited by $\Xbb$ have scale $\delta_k$. Thus, a remainder $R_i$ of smaller scale -- such that the length of the intervals defining $R_i$ is $\ll\delta_k$ -- will appear uniformly random to the first $k$ distinct inputs, similarly to the point-wise random sets from the partition $\Pcal$ introduced above.

\section{Main Result}
\label{section:main_result}
In this section, we state and prove the equivalence of $\fs$ and $\fmv$ therefore guaranteeing that memorization is an optimistically universal learning rule in the unbounded setup.  The following result represents the main technical contribution of this work.

\begin{theorem}[Main Result]
\label{thm:fs=fmv}
For any separable metric space $(\Xcal,\metric)$, $\normalfont \fmv = \fs$.
\end{theorem}

\noindent
Together with Proposition \ref{prop:fs_sufficient} and Theorem \ref{thm:fmv_necessary}, this result implies Theorem \ref{thm:caracterization_suol} and Theorem \ref{thm:memorization_opt_universal}. In this section, we prove the result for $\Xcal = [0,1]$, as it provides a direct and simple construction. The proof will then be generalised to all separable metric spaces in Section \ref{sec:separable}.

\begin{proposition}
\label{thm:circle}
If $\mathcal X = [0,1]$ with its usual topology, $\normalfont\fmv = \fs$.
\end{proposition}

\begin{proof}
The inclusion $\fs \subset \fmv$ is a direct observation, therefore we focus on proving that $\fmv \subset \fs$. Let $\mathbb X$ be a stochastic process which does not satisfy $\fs$. The goal is to construct a countable measurable partition $\{A_k\}_{k=1}^{\infty}$ of $\Xcal$ which disproves condition $\fmv$, i.e. such that $\{k\in \mathbb N: A_k\cap \mathbb{X} \neq \emptyset \}$ is infinite with nonzero probability. Denote by $\Acal$ the event that $\Xbb$ takes an infinite number of values, i.e. $\mathcal{A} = \{\#\Xbb=+\infty\}$. We have assumed that $\Pbb(\mathcal A)>0$ and will condition on $\Acal$ for the rest of the proof. For $k\in \Nbb$, define $N_k\in \mathbb N$ such that,
\begin{equation*}
    \mathbb P\left( \left. \#\Xbb_{\leq N_k} \geq \frac{1}{\mu_k^2} \; \right| \mathcal A\right)\geq 1-\frac{1}{2^{k+1}}, \quad \text{where}\quad \mu_k := \frac{1}{2^{k+1}}.
\end{equation*}
Note that $N_k$ is a deterministic quantity that only depends on the process $\Xbb$. It is well defined because $ \mathbb P\left(\left.\#\Xbb_{\leq N} \geq \frac{1}{\mu_k^2} \;\right| \mathcal A\right)\rightarrow_{N\rightarrow \infty} 1$ since $\Xbb$ takes an infinite number of values in $\Acal$. Now also define $0<\delta_k<\frac{1}{2^{k+1}}$ satisfying:
\begin{equation*}
    \mathbb P\left( \left. \min_{1\leq i, j \leq N_k,X_i\neq X_j} |X_i-X_j| > \delta_k \;\right| \mathcal A\right)\geq 1 - \frac{1}{2^{k+1}}.
\end{equation*}
Let $\mathcal E_k$ be the intersection of the two events above, we have by union bound: $
    \mathbb P\left( \mathcal E_k \mid \mathcal A\right)\geq 1 - \frac{1}{2^{k}}$, where $\Ecal_k$ can be written as
\begin{equation*}
    \Ecal_k:\quad \#\Xbb_{\leq N_k}>\frac{1}{\mu_k^2}  \quad  \text{and} \quad \forall x\neq y\in \Xbb_{\leq N_k}, ~|x-y|> \delta_k. 
\end{equation*}

\noindent
We are now ready to construct the partition. Let $\bq = (q_i)_{i\geq 1}$ be an i.i.d. sequence of independent uniforms sampled from $\mathcal U([0,1])$. Define $B_0 =[0,1]$ and for $k\geq 1$,
\begin{equation*}
    B_k =  \bigcup_{i_{k-1}< i \leq i_k} \left[ q_i-\frac{\delta_k}{2},q_i+\frac{\delta_k}{2}\right],
\end{equation*}
where $(i_k)_{k\geq 0}$ satisfies $i_0=0$ and $i_k = i_{k-1} +  \lceil \frac{\mu_k}{\delta_k} \rceil$. Note that the Borel measure of $B_k$ is roughly $\mu_k$: $\mu(B_k)\leq \mu_k+\delta_k$. We use the remainders $R_k = [0,1]\cap\bigcup_{l\geq k}B_l$ to define the partition $\{A_k\}_{k= -1}^\infty$ as follows:
\begin{equation*}
    A_k = R_k\setminus R_{k+1},\quad k\geq 0
\end{equation*}
and $A_{-1} = \bigcap_{k\geq 0}R_k$. The sets $\{A_k\}_{k= -1}^\infty$ define a proper partition of $\Xcal$ since $A_{-1}$ contains elements that appear infinitely often in $\{B_k\}_{k\geq 0}$ while for any $k\geq 0$, $A_k$ contains elements that appear for the last time in the sequence $\{B_l\}_{k\geq 0}$ in $B_k$. This covers the whole space $\Xcal$ because by construction $B_0=\Xcal$. The interest of this (random) construction lies in the two following lemmas.

\begin{lemma}
\label{lemma:intersect_B}
For any finite (deterministic) $S \subset [0,1]$ with $\#S>\frac{1}{\mu_k^2}$ and $\forall x\neq y \in S, |x-y|>\delta_k$,
\begin{equation*}
    \mathbb P(B_k\cap S = \emptyset) \leq e^{-2^{k+1}}.
\end{equation*}
\end{lemma}

\begin{lemma}
\label{lemma:notinA-1}
For any countable (deterministic) $S\subset [0,1]$, $A_{-1}\cap S = \emptyset, \quad (a.s.)$
\end{lemma}

\noindent
We will now show that with probability $\Pbb(\Acal)$, the partition $\{A_k\}_{k= -1}^{\infty}$ disproves the condition $\fmv$, in other terms that $\mathbb X$ visits an infinite number of sets of the partition. Recall that the randomness is now both in terms of the stochastic process $\Xbb$ and the partition generated from $\bq$. We have that,
\begin{equation*}
    \mathbb P(B_k\cap\Xbb = \emptyset \mid \mathcal A)\leq (1-\Pbb(\Ecal_k|\Acal)) + \mathbb P(B_k \cap \Xbb_{\leq N_k} = \emptyset \mid  \mathcal E_k,\mathcal A) \leq  \frac{1}{2^k}+e^{-2^{k+1}},
\end{equation*}
where in the last inequality we applied Lemma \ref{lemma:intersect_B} to the set $\Xbb_{\leq N_k}$ which has cardinality at least $\frac{1}{\mu_k^2}$ in $\Ecal_k$. We can now apply the first Borel-Cantelli lemma to the sequence of events $\{B_k\cap \Xbb=\emptyset\}$ conditionally on $\Acal$, which shows that almost surely only a finite number of these events are satisfied. Hence, conditionally on $\Acal$, there exists almost surely $\kappa\in \Nbb$ such that for every $k\geq \kappa$, the sequence $\Xbb$ visits $B_k$. Further, by Lemma \ref{lemma:notinA-1}, with probability $1$, $\mathbb X$ does not visit $A_{-1}$. Therefore, conditionnally on $\Acal$, the sequence almost surely visits an infinite number of sets of the partition $\{A_k\}_{k= -1}^\infty$. In summary,
$$\mathbb P_{\bq, \mathbb X}(\# \{k\in \mathbb N: A_k\cap \mathbb{X} \neq \emptyset \} = +\infty)\geq \Pbb(\Acal).$$
Thus, there exists a \emph{deterministic} choice of $\bq$ yielding a partition $\{A_k\}_{k=-1}^\infty$ such that:
\begin{equation*}
    \mathbb P_{\mathbb X}(\# \{k\in \mathbb N, A_k\cap \mathbb{X} \neq \emptyset \} = +\infty)\geq \Pbb(\Acal)>0.
\end{equation*}
This shows the claim of the theorem.
\end{proof}

\begin{proof}[of Lemma~\ref{lemma:intersect_B}]
Note that the randomness of $\mathbb{X}$ does not intervene in this lemma. The probability law $\Pbb$ only accounts for the randomness of the partition through the variables $\bq = (q_i)_{i\geq 1}$. We enumerate $S = \{x_1, ..., x_{T}\}$ where $T = \#S \geq \frac{1}{\mu_k^2}$.
\begin{equation*}
    \mathbb P(B_k\cap S = \emptyset) = \mathbb P(x_1\notin B_k) \prod_{t=2}^{T} \mathbb P(x_t\notin B_k |x_1,\ldots,x_{t-1}\notin B_k).
\end{equation*}
For the sake of simplicity we will use the notation $B(x,\delta) = [x-\delta,x+\delta]$. Note that events $\{x_i\notin B_k\}$ are negatively correlated. Indeed,
\begin{align*}
    \mathbb P(x_t\notin B_k \mid x_1,\ldots,x_{t-1}\notin B_k) &= \prod_{i=i_{k-1}+1}^{i_k} \mathbb P\left[q_i\notin B\left(x_t,\frac{\delta_k}{2}\right) \left| q_i\notin \bigcup_{1\leq l\leq t-1} B\left(x_l,\frac{\delta_k}{2}\right)\right.\right]\\
    &=\prod_{i=i_{k-1}+1}^{i_k} \mathbb P\left[\tilde q_i\notin B\left(x_t,\frac{\delta_k}{2}\right)\right]
\end{align*}
where $\tilde q_i \sim \mathcal U(J)$ with $J:=\Xcal \setminus \bigcup_{1\leq l\leq t-1} B(x_l,\frac{\delta_k}{2})$. Because  $|x_t-x_l|>\delta_k$ for all $1\leq l\leq t-1$, we have
$B(x_t, \frac{\delta_k}{2}) \subset J$. Thus,
\begin{equation*}
    \mathbb P(x_t\notin B_k |x_1,\ldots,x_{t-1}\notin B_k) = \prod_{i=i_{k-1}+1}^{i_k} \left(1-\frac{\delta_k}{\mu(J)} \right)\leq \prod_{i=i_{k-1}+1}^{i_k} (1-\delta_k) = \mathbb P(x_t\notin B_k).
\end{equation*}
Using the negative correlation, we have that
\begin{equation*}
    \mathbb P(B_k\cap S = \emptyset) \leq \prod_{t=1}^T\mathbb P(x_1\notin B_k) = (1-\delta_k)^{T(i_k-i_{k-1})} \leq (1-\delta_k)^{\frac{1}{\mu_k\delta_k}}\leq e^{-\frac{1}{\mu_k}} = e^{-2^{k+1}}.
\end{equation*}
This ends the proof of the lemma.
\end{proof}

\begin{proof}[of Lemma~\ref{lemma:notinA-1}]
We start by proving that for a given $x\in \Rbb$, $x \notin  A_{-1}$ a.s. For $k\geq 1$ we have,
\begin{equation*}
    \mathbb P(x\in B_k) \leq \left\lceil\frac{\mu_k}{\delta_k}\right\rceil \delta_k \leq\left(\frac{\mu_k}{\delta_k}+1\right) \delta_k  \leq \mu_k + \delta_k \leq \frac{1}{2^k}.
\end{equation*}
Therefore, $\mathbb P(x\in R_k)\leq \frac{1}{2^{k-1}}$. This shows that $\mathbb P(x\in A_{-1})\leq \mathbb P(x\in \bigcap_k R_k) = 0.$ Taking the union over all countable random variables in $S$, we have $\mathbb P(A_{-1}\cap S\neq \emptyset)=0$.
\end{proof}

\paragraph{Extension to all standard Borel spaces.} Before we move on to proving the main theorem in the most general framework of separable metric spaces, observe that the proof for $\Xcal = [0,1]$ easily extends to all standard Borel space by Kuratowski's theorem, in particular for instance to $\Xcal = \Rbb^d$. Kuratowski's theorem states that if $\Xcal$ is an uncountable standard Borel space it is isomorphic to $[0,1]$ with the Euclidean distance, meaning that there exists a measurable bijection $f:\Xcal \rightarrow [0,1]$. Let $\Xbb = (X_i)_{i\geq 0}$ be a stochastic process on $\mathcal X$ satisfying $\fmv$. Then, because $f$ is measurable, $\tilde \Xbb := (f(X_i))_{i\geq 0}$ is a stochastic process on $[0,1]$ which satisfies $\fmv$. By Theorem \ref{thm:circle}, $\tilde \Xbb$ satisfies $\fs$. Thus, because $f$ is bijective, $\Xbb$ also satisfies $\fs$.

\renewenvironment{proof}[1][]{\par\noindent{\bf Proof #1\ }}{\hfill\BlackBox\\[2mm]}

\renewcommand{\rm}{\mathrm}

\newcommand{\X}{\mathcal X} 
\newcommand{\Y}{\mathcal Y} 
\newcommand{\F}{\mathcal F} 
\newcommand{\G}{\mathcal G} 
\newcommand{\A}{\mathcal A} 
\renewcommand{\H}{\mathcal H} 
\renewcommand{\L}{\mathcal L} 
\newcommand{\U}{\mathcal U} 
\renewcommand{\P}{\mathbb P} 
\newcommand{\I}{\mathit I} 
\newcommand{\nats}{\mathbb{N}} 
\newcommand{\ints}{\mathbb{Z}} 
\newcommand{\reals}{\mathbb{R}} 
\newcommand{\Epsilon}{\mathscr{E}}
\newcommand{\E}{\mathbb E}
\newcommand{\eps}{\varepsilon}
\newcommand{\diam}{diam}
\newcommand{\argmax}{\mathop{\rm {argmax}}}
\newcommand{\argmin}{\mathop{\rm {argmin}}}
\renewcommand{\limsup}{\mathop{\rm {limsup}}}
\renewcommand{\liminf}{\mathop{\rm {liminf}}}
\newcommand{\Borel}{{\cal B}}
\newcommand{\ProcX}{\mathbb{X}}
\newcommand{\ProcY}{\mathbb{Y}}
\newcommand{\loss}{\ell}
\newcommand{\maxloss}{\bar{\loss}}
\newcommand{\triconst}{c_{\loss}}
\newcommand{\Loss}{\mathcal{L}}
\newcommand{\SUIL}{{\rm {SUIL}}}
\newcommand{\WUIL}{{\rm {WUIL}}}
\newcommand{\SUAL}{{\rm {SUAL}}}
\newcommand{\WUAL}{{\rm {WUAL}}}
\newcommand{\SUOL}{{\rm {SUOL}}}
\newcommand{\WUOL}{{\rm {WUOL}}}
\newcommand{\ProcSet}{\mathcal{C}}
\renewcommand{\metric}{\rho}
\renewcommand{\i}{{\mathrm{i}}}
\renewcommand{\j}{{\mathrm{j}}}

\newcommand{\ignore}[1]{}
\newcommand{\private}[1]{}

\section{Extension to General Separable Metric Spaces}
\label{sec:separable}

The original proof of $\suol = \fmv$ by \citep*{hanneke2021learning} 
holds for any separable metric space $(\X,\metric)$.  In this 
section, we extend the proof above to hold in this more-general case as 
well, thus completely answering the question (Open Problem 4) posed by \citep*{hanneke2021learning} 
in full generality, and completing the proof of Theorems~\ref{thm:caracterization_suol} and \ref{thm:memorization_opt_universal}.  For the remainder of this section, we let $(\X,\metric)$ denote 
a non-empty separable metric space, and we take as the set $\Borel$ of 
measurable subsets of $\X$ the Borel $\sigma$-algebra generated by 
the topology induced by $\metric$.\\

\noindent \textbf{Theorem~\ref{thm:fs=fmv} (Restated)}~~
For any separable metric space $(\X,\metric)$, 
$\normalfont \fmv = \fs$.\\

\noindent
The main components of the proof are analogous to those for standard Borel spaces, 
with a few important changes: most importantly, the following lemma.

\begin{lemma}
\label{lem:effectively-totally-bounded}
For any $\ProcX$ satisfying condition ${\normalfont \fmv}$, for any $\delta,\eps > 0$ and $m_0 \in \nats$, 
there exists $M_{\eps,\delta} \in \nats$ with $M_{\eps,\delta} \geq m_0$, and a sequence 
$\G^{\eps,\delta} = \{ G^{\eps,\delta}_{1},\ldots,G^{\eps,\delta}_{M_{\eps,\delta}} \}$ in $\Borel$ 
such that every distinct $i,j \in \{1,\ldots,M_{\eps,\delta}\}$ satisfy $G^{\eps,\delta}_{i} \cap G^{\eps,\delta}_{j} = \emptyset$, 
and every $i \in \{1,\ldots,M_{\eps,\delta}\}$ satisfies $\sup_{x,x' \in G^{\eps,\delta}_{i}} \metric(x,x') \leq \delta$, 
and such that
\begin{equation*}
\P\!\left( \ProcX \cap \left( \X \setminus \bigcup_{i=1}^{M_{\eps,\delta}} G^{\eps,\delta}_{i} \right) \neq \emptyset \right) < \eps.
\end{equation*}
In other words, $\G^{\eps,\delta}$ is a sequence of disjoint measurable sets 
of diameter at most $\delta$, 
which cover all of the points in $\ProcX$ with probability $1-\eps$.
\end{lemma}
\begin{proof}
Let $\tilde{\X} \subseteq \X$ be a countable dense subset: that is, $\sup_{x \in \X} \inf_{\tilde{x} \in \tilde{\X}} \metric(\tilde{x},x) = 0$.
Enumerate $\tilde{\X}$ as $\{\tilde{x}_1,\tilde{x}_2,\ldots,\}$.
Let $G^{\eps,\delta}_{1} = \{ x : \metric(x,\tilde{x}_1) \leq \delta/2 \}$, 
and for integers $k \geq 2$ inductively define 
$G^{\eps,\delta}_{k} = \{ x : \metric(x,\tilde{x}_{k}) \leq \delta/2 \} \setminus \bigcup_{k' = 1}^{k-1} G^{\eps,\delta}_{k'}$.
In particular, this collection $\{ G^{\eps,\delta}_{k} : k \in \nats \}$ forms a countable partition of $\X$ 
into measurable subsets of diameter at most $\delta$ (by the triangle inequality).
Now let $\ProcX$ be any process satisfying $\fmv$.
It remains only to show there exists a finite $M_{\eps,\delta} \in \nats$ satisfying the claim.
Let $\hat{M} = \max\{ k : \ProcX \cap G^{\eps,\delta}_{k} \neq \emptyset \}$, or $\hat{M} = \infty$ if there is no maximum.
By hypothesis, $\P( \hat{M} < \infty ) = 1$.
Since the event $\{ \hat{M} > M \}$ is non-increasing in $M$, 
$\lim_{M \to \infty} \P( \hat{M} > M ) = \P( \hat{M} = \infty ) = 0$.  
Thus, $\exists M_{\eps,\delta} \in \nats$ with $M_{\eps,\delta} \geq m_0$ such that 
$\P(\hat{M} > M_{\eps,\delta}) < \eps$.
In other words, $\P( \exists k > M_{\eps,\delta} : \ProcX \cap G^{\eps,\delta}_{k} \neq \emptyset ) < \eps$.
Since $\{ G^{\eps,\delta}_{k} : k \in \nats \}$ is a partition of $\X$, 
this implies the claim in the lemma.
\end{proof}

We are now ready for the main proof.\\

\begin{proof}[of Theorem~\ref{thm:fs=fmv}]
Since condition $\fs$ clearly implies condition $\fmv$\comment{(via the equivalence in Lemma 49)}, 
we focus on showing $\fmv\subset\fs$.
Let $\ProcX$ be any process satisfying condition $\fmv$, 
and for the sake of obtaining a contradiction, 
suppose that condition $\fs$ fails: 
that is, there is an event $\A$ with $\P(\A) > 0$, 
on which $\#\{ x \in \X : \ProcX \cap \{x\} \neq \emptyset\} = \infty$.

For each $k \in \nats$, let $N_k \in \nats$ be such that 
\begin{equation*}
\P\!\left( \#\ProcX_{\leq N_k} \geq 2^{2k+2} \middle| \A \right) \geq 1-\frac{1}{2^{k+2}},
\end{equation*}
and let $\delta_{k} > 0$ be such that 
\begin{equation*}
\P\!\left( \min_{i,j \leq N_k : X_i \neq X_j} \metric(X_i,X_j) > \delta_k \middle| \A \right) \geq 1 - \frac{1}{2^{k+3}}.
\end{equation*}
Let $S_k = \{ x \in \X : \ProcX_{\leq N_k} \cap \{x\} \neq \emptyset \}$
and let $\eps_k = \frac{1}{2^{k+3}}$.
Let $\G^{\eps_k,\delta_k}$ and $M_{\eps_k,\delta_k}$ be as in Lemma~\ref{lem:effectively-totally-bounded}, 
with $m_0 = 2^{k+2}$.

Let $\Epsilon_k$ denote the event that  
$\#\ProcX_{\leq N_k} \geq 2^{2k+2}$, 
$\min_{i,j \leq N_k : X_i \neq X_j} \metric(X_i,X_j) > \delta_k$, 
and 
$\ProcX \cap \left( \X \setminus \bigcup \G^{\eps_k,\delta_k} \right) = \emptyset$ 
all hold simultaneously.  In particular, by the union bound, $\P( \Epsilon_k | \A ) \geq 1 - 2^{-k-1}$.

For each $k \in \nats$,
let $b_k = \left\lceil 2^{-k-2} M_{\eps_k,\delta_k} \right\rceil$, 
and let $Q_1^k,\ldots,Q_{b_k}^k$ be independent uniform samples from $\G^{\eps_k,\delta_k}$ (also independent across $k$ and independent from $\ProcX$).
Then let $B_k = \bigcup_{i=1}^{b_k} Q_i^k$.
For each $k \in \nats$, 
let $R_k = \bigcup_{\ell \geq k} B_{\ell}$.
Also let $A_{-1} = \bigcap_{k \in \nats} R_k$ 
and for each $k \in \nats$, let $A_k = R_k \setminus R_{k+1}$,
and $A_0 = \X \setminus R_1$.
We will show that (with non-zero probability) 
the countable measurable partition 
$\{A_k : k \in \nats \cup \{-1,0\}\}$ 
violates the condition $\fmv$, thus obtaining a contradiction.

Now note that, on the event $\Epsilon_k$, 
every $x \in S_k$ is in a distinct set $G^{\eps_k,\delta_k}_{i} \in \G^{\eps_k,\delta_k}$: 
that is, by definition of $\Epsilon_k$, every $x \in S_k$ is in some $G^{\eps_k,\delta_k}_{i} \in \G^{\eps_k,\delta_k}$, 
and since each $G^{\eps_k,\delta_k}_{i}$ has diameter at most $\delta_k$, while 
every distinct $x,x' \in S_k$ are $\delta_k$-separated (on event $\Epsilon_k$),
no two elements of $S_k$ can be in the same $G^{\eps_k,\delta_k}_{i}$.
Therefore, on the event $\Epsilon_k$ we have that 
\begin{align*}
\P\!\left( B_k \cap S_k = \emptyset \middle| \ProcX \right) 
= \P\!\left( Q_1^k \cap S_k = \emptyset \middle| \ProcX \right)^{b_k} 
= \left( 1 - \frac{|S_k|}{M_{\eps_k,\delta_k}} \right)^{b_k}
\leq e^{- |S_k| b_k / M_{\eps_k,\delta_k}}
\leq e^{- 2^{k} },
\end{align*}
where the last inequality is based on the definition of $b_k$ and 
the fact that $|S_k| \geq 2^{2k+2}$ on the event $\Epsilon_k$.
Thus, on the event $\bigcap_{k \in \nats} \Epsilon_k$, 
\begin{equation*}
\sum_{k=1}^{\infty} \P\!\left( B_k \cap S_k = \emptyset \middle| \ProcX \right) 
\leq \sum_{k=1}^{\infty} e^{-2^k} < \infty.
\end{equation*}
By the Borel-Cantelli lemma, this implies that there is an event $\Epsilon^{\prime}$ 
of probability one, such that on $\Epsilon^{\prime} \cap \bigcap_{k \in \nats} \Epsilon_k$, 
there exists $\kappa \in \nats$ such that 
every $k \geq \kappa$ satisfies  
$B_k \cap S_k \neq \emptyset$, 
and hence also $\ProcX \cap R_k \neq \emptyset$.
Now, if $\ProcX \cap A_{-1} = \emptyset$, this would further imply 
that $|\{ k \in \nats : \ProcX \cap A_k \neq \emptyset \}| = \infty$.

We next turn to showing that $\ProcX \cap A_{-1} = \emptyset$ (a.s.).
For any $t,k \in \nats$, by the union bound, 
$\P(X_t \in B_k) \leq \frac{b_k}{M_{\eps_k,\delta_k}} \leq 2^{-k-1}$ 
(recalling that $M_{\eps_k,\delta_k} \geq 2^{k+2}$, so that $b_k \leq 2^{-k-1} M_{\eps_k,\delta_k}$).
By the union bound, this further implies any $t,k \in \nats$ satisfy 
$\P(X_t \in R_k) \leq \sum_{\ell \geq k} \P(X_t \in B_{\ell}) \leq \sum_{\ell \geq k} 2^{-\ell-1} = 2^{-k}$.
Thus, $\P(X_t \in A_{-1}) = \P\!\left(X_t \in \bigcap_{k \in \nats} R_k\right) \leq \lim_{k \to \infty} \P(X_t \in R_k) = 0$.
By the union bound, $\P( \ProcX \cap A_{-1} \neq \emptyset ) = 0$.
Thus, there is an event $\Epsilon^{\prime\prime}$ of probability one, on which $\ProcX \cap A_{-1} = \emptyset$.

Altogether, we have that on the event $\Epsilon^{\prime} \cap \Epsilon^{\prime\prime} \cap \bigcap_{k \in \nats} \Epsilon_k$, 
$|\{ k \in \nats : \ProcX \cap A_k \neq \emptyset \}| = \infty$.
Since $\P(\Epsilon^{\prime}) = \P(\Epsilon^{\prime\prime}) = 1$, 
and 
\begin{align*}
\P\!\left( \bigcap_{k \in \nats} \Epsilon_k \right) 
& \geq \P\!\left(  \A \cap \bigcap_{k \in \nats} \Epsilon_k \right) 
\geq \P(\A) - \sum_{k \in \nats} \P(\A) \left( 1 - \P( \Epsilon_k | \A ) \right)
\\ & \geq \P(\A) - \sum_{k \in \nats} \P(\A) 2^{-k-1}
= \frac{1}{2}\P(\A),
\end{align*}
by the union bound we have 
$\P\!\left( \Epsilon^{\prime} \cap \Epsilon^{\prime\prime} \cap \bigcap_{k \in \nats} \Epsilon_k \right) \geq \frac{1}{2} \P(\A) > 0$.
In particular, this implies 
$\P\!\left( |\{ k \in \nats : \ProcX \cap A_k \neq \emptyset \}| = \infty \right) > 0$.
Moreover, by the law of total probability, 
\begin{equation*}
\P\!\left( |\{ k \in \nats : \ProcX \cap A_k \neq \emptyset \}| = \infty \right) 
= \E\!\left[ \P\!\left( |\{ k \in \nats : \ProcX \cap A_k \neq \emptyset \}| = \infty \Big| \{A_k : k \in \nats\} \right) \right],
\end{equation*}
and hence (since $\ProcX$ is independent of the random partition $\{A_k : k \in \nats \cup \{-1,0\}\}$), 
there exists a \emph{deterministic} choice of a partition $\{ \hat{A}_k : k \in \nats \cup \{-1,0\} \}$ such that 
\begin{equation*}
\P\!\left( |\{ k \in \nats \cup \{-1,0\} : \ProcX \cap \hat{A}_k \neq \emptyset \}| = \infty \right) > 0,
\end{equation*}
contradicting \comment{Lemma 49}condition $\fmv$.  This completes the proof.
\end{proof}



\section{Consequences on inductive and self-adaptive learning}\label{section:consequences}
Along with optimistically universal \textit{online} learning, \cite{hanneke2021learning} identifies two other learning setups, namely \emph{inductive learning} and \emph{self-adaptive learning}.

\paragraph{Inductive learning.} An inductive learning rule $\{f_t\}_{t=1}^{\infty}$ is a sequence of measurable functions $f_t: \Xcal^{t-1}\times \Ycal^{t-1}\times \Xcal \rightarrow \Ycal$ such that given training data $(\Xbb_{<t},\Ybb_{<t})$ and input point $X_{t'}$ with $t'>t$ outputs prediction $f_t(\Xbb_{<t},\Ybb_{<t},X_{t'}).$ Its performance is measured in terms of,
\begin{equation*}
    \mathcal{L}_{\Xbb}(f_{t},f^*;t) = \limsup_{T\rightarrow\infty} \frac{1}{T}\sum_{t'=t}^{t+T}\ell(f_t(\Xbb_{<t},\Ybb_{<t},X_{t'}), f^*(X_{t'})).
\end{equation*}
Let $\suil$ denote the set of all processes $\Xbb$ that admit strong universal inductive learning: i.e., for which there exists an inductive learning rule $\{f_t\}$ such that for every measurable $f^* : \Xcal \to \Ycal$, $\mathcal{L}_{\Xbb}(f_{t},f^*;t) \to 0~~(a.s.)$. Note that the difference between an online learning rule and its inductive counterpart is that the latter will be fixed for an infinite horizon. It was therefore shown by \cite{hanneke2021learning} that $\suil\subset \suol$.

\paragraph{Self adaptive learning rule.} A self-adaptive learning rule $\{f_{t_1,t_2}\}_{t_1\leq t_2}^{\infty}$ is a sequence of measurable functions $f_{t_1,t_2}: \Xcal^{t_2-1}\times \Ycal^{t_1-1}\times \Xcal \rightarrow \Ycal$ such that given training data $(\Xbb_{<t_2},\Ybb_{<t_1})$ and input point $X_{t_2}$ it performs prediction $f_{t_1,t_2}(\Xbb_{<t_2},\Ybb_{<t_1},X_{t_2}).$ Its performance is measured in terms of
\begin{equation*}
    \mathcal{L}_{\Xbb}(f_{t_1, \cdot},f^*;t_1) = \limsup_{T\rightarrow\infty} \frac{1}{T}\sum_{t_2=t_1}^{t_1+T}\ell(f_{t_1,t_2}(\Xbb_{<t_2},\Ybb_{<t_1},X_{t_2}), f^*(X_{t_2})).
\end{equation*}
Let $\sual$ denote the set of all processes $\Xbb$ that admit strong universal self-adaptive learning: i.e., for which there exists a self-adaptive learning rule $\{f_{t_1,t_2}\}$ such that for every measurable $f^* : \Xcal \to \Ycal$, $\mathcal{L}_{\Xbb}(f_{t_1,\cdot},f^*;t_1) \to 0~~(a.s.)$. Note that self-adaptive learning rules are more expressive than inductive learning rules for they have access to additional unlabeled data, like in the semi-supervised learning setup studied in the literature \cite{chapelle2009semi}, yet are still less powerful than online learning rules which would also have access to the respective labels. It was therefore shown by \cite{hanneke2021learning} that $\suil\subset \sual \subset\suol$.

\paragraph{Consequence of the Main Theorem.} For unbounded losses, \cite{hanneke2021learning} shows that $\suil = \sual = \suol$. However, once again this proof relied on the aforementioned complicated arguments.  But in light of our proof that $\suol = \fs$, it becomes immediately apparent that $\suil=\sual=\suol$ and that these classes all admit memorization as an optimistically universal learning rule (merely noting that $\fs \subset \suil$, since for any $\Xbb \in \fs$, the inductive loss of the memorization rule $f_t$ becomes zero once $t$ exceeds the index of the last novel data point). This greatly simplifies the proof of these equivalences compared to the original proof of \cite{hanneke2021learning}. Note that while these three setups turn out to be equivalent when the loss is unbounded, interesting distinctions do exist in the bounded case for which \cite{hanneke2021learning} proved that there exists an optimistically universal self-adaptive learning rule (which surprisingly is necessarily different from nearest-neighbour), but no optimistically universal inductive learning rule. 

\section{Discussion on noise : an optimistically universal Bayes consistent learner}\label{sec:bayesian}
For simplicity we restricted the analysis to the \emph{realizable} setting \cite{littlestone1988learning,ben2009agnostic} where there exists a measurable function $f^*$ satisfying $\forall t\geq 1: Y_t = f^*(X_t)$. A common variant allows for the function $f^*$ to be \emph{noisy}, i.e. to take  the form of a conditional probability density $p_{Y|X}$. In this section, we will generalise the main results to this context. We start by recalling the adequate notion of consistency, e.g. \cite{hanneke2021learning,hanneke2020universal,cohen2022metric}. We say that the learning rule $f_{\cdot}$ is strongly universally Bayes consistent under $\Xbb$ if for all conditional probability distribution $p_{Y|X}$ and any measurable function $\fb$, almost surely 
\begin{equation*}
    \mathcal{L}_{\Xbb}(f_{\cdot},p_{Y|X};T,\fb) \leq 0 
\end{equation*}
where $\hat{\mathcal{L}}_{\Xbb}(f_{\cdot},p_{Y|X};T,\fb)$ is the excess loss of the learning rule $f_{\cdot}$ against the constant predictor $\bar{f}$,
\begin{equation*}
    \mathcal{L}_{\Xbb}(f_{\cdot},p_{Y|X};T,\fb) = \limsup_{T\rightarrow\infty}\frac{1}{T}\sum_{t=1}^T\left(\ell(f_t(\Xbb_{<t},\Ybb_{<t},X_t),Y_t)-\ell(\bar{f}(X_t),Y_t)\right).
\end{equation*}
\begin{theorem}
\label{thm:bayes_optimistic}
If $(\Ycal, d)$ is a separable locally compact metric space with $\bar{d}=\infty$, there exists an optimistically universal Bayes consistent learning rule for the loss $\ell=d^p$, for any $p\geq 1$. 
\end{theorem}
Note that the Bayesian setting is more general than the \emph{realizable} setting. Thus, any process $\Xbb$ that admits a universally Bayes consistent learning rule is in $\suol$, and therefore it takes a finite number of values almost surely. To prove Theorem \ref{thm:bayes_optimistic}, it will suffice to define a learning rule that is universally Bayes consistent with any such process. For this purpose, we use a result from \cite{evans2020strong} which we slightly adapt to our setting. The proof of the following theorem can be found in Appendix \ref{appendix:proof_evans}.
\begin{theorem}[\cite{evans2020strong}]\label{thm:evans}
Let $p\geq 1$ and $(\Ycal,d)$ be a separable metric space such that any closed ball is compact. Let $(Y_i)_{i\geq 1}$ be an i.i.d. sequence of random variables in $\Ycal$ of distribution $Y$ satisfying $\Ebb d^p(y_0,\Ycal)<\infty$ for some $y_0\in \Ycal$. Denote $\hat y_n$ a Fr\'echet sample mean of the samples $Y_1,\ldots,Y_n$. Then,
\begin{equation*}
    \frac{1}{n}\sum_{i=1}^n d^p(\hat y_n,Y_i)\to \min_{y\in \Ycal} \Ebb d^p(y,Y)\quad (a.s.).
\end{equation*}
\end{theorem}
In the equation above, we use the notion of Fréchet sample mean that is defined as follows,
\begin{equation*}
    \hat{y}_n \in \argmin_{y\in \Ycal} \frac{1}{n}\sum_{i=1}^n d^p(y,Y_i).
\end{equation*}
Note that the minimum is well defined because the closed balls in the space $(\Ycal,d)$ are compact and $d^p(y,y_0)\leq 2^{p-1} \frac{1}{n}\sum_{i=1}^n (d^p(y,Y_i) + d^p(y_0,Y_i))$, hence $\frac{1}{n}\sum_{i=1}^n d^p(y,Y_i) \geq 2^{1-p}d^p(y,y_0) - \frac{1}{n}\sum_{i=1}^n d^p (y_0,Y_i)$ for any $y_0 \in \Ycal$. Therefore the expression is minimized in the (compact) closed ball of radius at most $2\left[\frac{1}{n}\sum_{i=1}^n d^p (y_0,Y_i)\right]^{1/p}$ around $y_0$ i.e. $d^p (\hat y_n,y_0)\leq \frac{2^p}{n}\sum_{i=1}^n d^p (y_0,Y_i)$. Similarly, the infimum $\inf_{y\in\Ycal}\Ebb d^p(y,Y)$ is attained because for $y_0\in\Ycal$ such that $\Ebb d^p(y_0,Y)<\infty$, the quantity $\Ebb d^p(y,Y)$ is minimized in the compact closed ball around $y_0$ of radius $2[\Ebb d^p (y_0,Y)]^{1/p}$.\\

\noindent
The discussion above allows to define the \emph{Fréchet mean memorizer} learning rule, which is proved to be an optimistically universal Bayes consistent learner in Appendix \ref{appendix:proof_bayes}. Fix an arbitrary $y_0\in\Ycal$,
\begin{equation*}
    f_{t}(\mb x_{< t}, \mb  y_{<t },x_{t}) = \begin{cases}
    y \in \argmin_{y\in \Ycal} \sum_{i=1}^{t-1} \1_{x_i=x_t}\ell(y,y_i), &\text{if }x_t\in \mb x_{<t},\\
    y_0 &\text{if } x_t\not\in \mb x_{<t}.
    \end{cases}
\end{equation*}
Note that the discussion above covers the case of real regression with  $\Ycal=\Rbb$ and $\ell(y_1,y_2) = (y_1-y_2)^2$. In this case, the Frechet sample mean from Theorem \ref{thm:evans} is the empirical average $\hat y_n=\frac{1}{n}\sum_{i=1}^n Y_i$ and the Bayes consistency of the corresponding learning rule can be obtained directly from the strong law of large numbers. In the setup where $\Xbb$ is an iid process, the question of Bayes consistency is addressed by \cite{hanneke2020universal,cohen2022metric}. In particular, \cite{cohen2022metric} uses the notion of medoid that adapts Fréchet means to compression-based algorithms. 

\comment{
Note that Theorem \ref{thm:bayes_optimistic} is limited to $(\Ycal, d)$ being a (separable locally compact) metric space, thus the case of real regression $\Ycal = \Rbb$ with squared loss $\ell(y_1,y_2) = (y_1-y_2)^2$ does not fall in this result. However, it is easy to note that for this setting, the \emph{empirical average memorizer} will also be optimistically universal as a consequence of the strong law of large numbers
\begin{equation*}
    f_{t}(\mb x_{< t}, \mb  y_{<t },x_{t}) = \begin{cases}
    \frac{1}{\sum_{i=1}^{t-1} \1_{x_i=x_t}}\sum_{i=1}^{t-1} \1_{x_i=x_t}y_i, &\text{if }x_t\in \mb x_{<t},\\
    y_0 \text{ if } x_t\not\in \mb x_{<t}.
    \end{cases}
\end{equation*}
}

\section{Conclusion and Open Directions}\label{sec: conclusion}
In this paper, we showed that memorization is an optimistically universal learning rule when the loss is unbounded. This closes the study of unrestricted universal consistency with unbounded losses for the \textit{online},  the \textit{inductive} and the \textit{self-adaptive} setups. In a sense, our result may be viewed as a \emph{negative} result, revealing that for unbounded losses, the processes in $\suol$ are all, to some extent, rather trivial.  On the other hand, we know of many positive results for universal consistency with unbounded losses for i.i.d.\ or stationary ergodic processes, under \emph{additional conditions} on the $Y_t$ sequence, such as with \emph{moment} conditions on $Y_t$ in the regression setting \cite{gyorfi:02,gyorfi:07}.  Thus, it would seem the next chapter in the study of universal consistency with unbounded losses and general non-i.i.d.\ families of processes should be to formulate broad sufficient conditions on the $Y_t$ sequence (relative to the given $X_t$ sequence) so that the family of processes $\Xbb$ admitting universal learning becomes rich, and in particular, includes within it all i.i.d.\ or stationary ergodic processes $\Xbb$.  It would be particularly interesting if there is a moment condition on the $Y_t$ sequence (or more-generally, on the empirical moments of $\ell(y_0,Y_t)$ for some $y_0$), under which the set of all processes $\Xbb$ admitting strong universal learning are precisely the same as for the case of \emph{bounded} losses: i.e., in the case of online learning, the set $\SUOL$ that would result from learning with a bounded loss (see \cite{hanneke2021learning,hanneke2021open} for discussions regarding this set), or in the case of inductive or self-adaptive learning, the sets $\SUIL$ or $\SUAL$, respectively, that would result from learning with a bounded loss (which have been characterized by \cite{hanneke2021learning}).

{\vskip 2mm}For concreteness, focusing on the setting of online learning, and letting $\suol_{{\scriptscriptstyle{01}}}$ denote the set of processes $\Xbb$ 
that admit strong universal online learning under the $0$-$1$ loss for 
binary classification (i.e., $\Ycal = \{0,1\}$ and $\ell(y,y') = \one[ y \neq y' ]$), we ask the following question:

{\vskip 4mm}\noindent \textbf{Open Problem:} For every unbounded loss $\ell$, 
does there exist an online learning rule $\{f_t\}_{t=1}^{\infty}$ with the 
property that, for every $\Xbb \in \suol_{{\scriptscriptstyle{01}}}$, we have 
$\mathcal{L}_{\Xbb}(f_{\cdot},f^*) = 0~~\text{(a.s.)}$ 
for every measurable $f^* : \Xcal \to \Ycal$ such that, for every $y_0 \in \Ycal$, 
$\limsup_{T \to \infty} \frac{1}{T} \sum_{t=1}^{T} \ell(y_0, f^*(X_t)) < \infty~~\text{(a.s.)}$?  In other words, $f^*$ has empirically bounded losses in the long-run average.

{\vskip 4mm}\noindent If this is found to be true, it would generalize the known conditions on consistent regression (for deterministic functions) with the squared loss under i.i.d.\ processes with finite $\mathrm{Var}(Y)$ \cite{gyorfi:02}. 
Of course, even variations of the above problem using milder restrictions on $(\Xbb,f^*)$ would be interesting.  As such, we may essentially pose a relaxed version of the question, which replaces the last condition with the mere requirement that, 
for every $y_0 \in \Ycal$, 
$\limsup_{T \to \infty} \frac{1}{T} \sum_{t=1}^{T} \ell(y_0, f^*(X_t))^p < \infty~~\text{(a.s.)}$, 
where $p > 0$ is some constant: i.e., any limiting empirical \emph{moment} condition.\\

In the case of \emph{bounded} losses, both the question of the existence of optimistically universal online learning rules, and of concisely characterizing the set $\suol$, remain open, as recently highlighted in a COLT open problem article \cite{hanneke2021open}. In particular, \cite{hanneke2021open} conjectures that for bounded losses, a process $\Xbb$ is in $\suol$ if and only if it has the property that, for any countable measurable partition $\{\A_k\}_{k=1}^{\infty}$ of $\Xcal$, the number of visited sets grows sub-linearly with time: $\#\{ k\in\Nbb : A_k \cap \Xbb_{\leq T}\} = o(T)$~~(a.s.).
The success of memorization in the unbounded setup suggests that a simple rule such as nearest neighbour might possibly be optimistically universal for online learning with bounded losses. However, such intuitions can also be misleading. In the self-adaptative setup with bounded losses, \cite{hanneke2021learning} proved that nearest neighbour is not optimistically universal although another more-intricate learning rule is optimistically universal.\\

\acks{M. Blanchard was partly funded by ONR grant N00014-18-1-2122.}

\bibliography{refs}

\appendix

\section{Proof of Proposition \ref{prop:no_test_fs}}
\label{appendix:no_consistent_test_fs}

Since $\Xcal$ infinite, let $\{x_i\}_{i\geq 0}$ a sequence of distinct points of $\Xcal$. Let $\{\hat t_n\}_{n\geq 0}$ be a hypothesis test for condition $\fs$. We suppose by contradiction that $\{\hat t_n\}$ is consistent and aim to construct a sequence $\Xbb$ on which it fails. Following a proof construction introduced in \cite{hanneke2021learning} (Theorem 47), we construct a (deterministic) process which fools the test by alternatively switching between two modes: constant $X_t = x_0$ or visiting points of the distinct sequence $X_t = x_t$. Let $n_0=0$ and $X_0=x_0$. We construct the sequences $\Xbb$ and $(n_i)_{i\geq 0}$ by induction. Suppose we have constructed $n_t$ for $0\leq t\leq k-1$ and $X_t$ for $0\leq t\leq n_{k-1}$.
\begin{itemize}
    \item If $k$ is even, consider the deterministic process $\Ybb$ such that $Y_t=X_t$ for $t\leq n_{k-1}$ and $Y_t=x_0$ for $t>n_{k-1}$. Because $\{\hat t_n\}$ is consistent, we can define an index $n_k>n_{k-1}$ such that $\Pbb(\hat t_{n_k}(\Ybb_{\leq n_k})= 1)>\frac{3}{4}$.
    \item If $k$ odd, consider the deterministic process $\Ybb$ such that $Y_t=X_t$ for $t\leq n_{k-1}$ and $Y_t=x_t$ for $t>n_{k-1}$. Similarly, let $n_k>n_{k-1}$ such that $\Pbb(\hat t_{n_k}(\Ybb_{\leq n_k})= 0)>\frac{3}{4}$.
\end{itemize}
We then set $X_t=Y_t$ for $n_{k-1}<n_k$. Note that for all $k\geq 0$, $\Pbb(\hat t_{n_{2k}}(\Xbb_{\leq n_{2k}})= 1)>\frac{3}{4}$ and $\Pbb(\hat t_{n_{2k+1}}(\Xbb_{\leq n_{2k+1}})= 1)<\frac{1}{4}$. Then, $\hat t_n(\Xbb_{\leq n})$ does not converge in probability and the hypothesis test $\{\hat t_n\}$ is not consistent. This ends the proof of the proposition.

\section{Proof of Theorem \ref{thm:fmv_necessary}}
\label{appendix:proof_fmv_necessary}

Let $\Xbb$ be a stochastic process that does not satisfy $\fmv$ and $f_n$ be a learning rule, we aim to show that this learning rule cannot be universally consistent. By hypothesis, there exists a finite measurable partition $\{A_k\}_{k=1}^{\infty}$ in $\Bcal$ such that $\Xbb$ visits an infinite number of the $A_k$ with probability $p>0$. We denote by $\Acal$ this event. We call $\Fcal$ the class of measurable functions $f:\Xcal\rightarrow\Ycal$ that takes constant values on each of the $A_k$. We will show that some objective function $f^*\in \Fcal$ cannot be learnt by $f_n$.\\

\noindent
First, let us define $\tau_k$ the first instant at which $\Xbb$ attains $A_k$.
$$\tau_k = \begin{cases}
\min\{t\in \Nbb:X_t \in A_k\} &\text{ if }  A_k\cap \Xbb\neq \emptyset
\\0 &\text{ otherwise.}
\end{cases}$$
We also define a deterministic quantity $T_k \in \Nbb$ that upper bounds the $\tau_k$ with high probability, i.e.
$$\Pbb(\tau_k \leq T_k)>1-2^{-k}, \quad \forall k\geq 1.$$
By the Borel Cantelli lemma, since $\sum_k \Pbb(\tau_k>T_k)<\infty$, almost surely there exists $\kappa\in \Nbb$ such that $\tau_k\leq T_k$ for $k\geq \kappa$. We will denote $\Ecal$ this event.
We now sample $f^*$ randomly from $\Fcal$ as follows:
$$f^*(x\in A_k) = \begin{cases} 
y_{k,0} &\text{ with proba } 1/2,\\
y_{k,1} &\text{ with proba } 1/2,\end{cases}$$
where $y_{k,1}$ and $y_{k,0}$ are selected such that $ \ell(y_{k,1},y_{k,0})\geq 2 c_{\ell} T_k$ (recalling that $c_{\ell}$ denotes the constant from the relaxed triangle inequality satisfied by $\ell$).
Note that taking the expectation over the randomness in $f^*$ allows to write:
\begin{equation}\label{eq:1}
    \sup_{g\in \Fcal}\Ebb_{\Xbb}(\Lcal_{\Xbb}(f_{.},g))\geq \Ebb_{f^*,\Xbb}(\Lcal_{\Xbb}(f_{.},f^*)).
\end{equation}
We first prove a lower bound on the right term. Conditionally on $\Acal\cap\Ecal$, observe that for any $k\geq \kappa$,  $(\Xbb_{<\tau_k},\Ybb_{<\tau_k})$ provides no information on $f^*(X_{\tau_k})$. Then, the average of the corresponding prediction error satisfies $\Ebb(\ell(f_{\tau_k}(\Xbb_{<\tau_k},\Ybb_{<\tau_k},X_{\tau_k}),f^*(X_{\tau_k}))\geq T_k\geq \tau_k$, where we used the fact that $\ell$ satisfies the relaxed triangle inequality. Thus, in $\Acal\cap\Ecal$, $\Ebb_{f^*}(\Lcal_{\Xbb}(f_{.},f^*,\tau_k))\geq 1$ for any $k\geq \kappa$, hence by Fatou's lemma
\begin{equation*}
    \Ebb_{f^*}(\Lcal_{\Xbb}(f_{.},f^*))\geq \limsup_{t\in \Nbb} \Ebb_{f^*}(\Lcal_{\Xbb}(f_{.},f^*,\tau_k))\geq 1.
\end{equation*}
Therefore, $\Pbb_{\Xbb}(\Ebb_{f^*}(\Lcal_{\Xbb}(f_{.},f^*)) \geq 1)\geq \Pbb(\Acal\cap\Ecal)=p$, which yields $\Ebb_{f^*,\Xbb}(\Lcal_{\Xbb}(f_{.},f^*))\geq p$.
Equation \ref{eq:1} then shows that there exists $g\in \Fcal$ such that $\Ebb_{\Xbb}(\Lcal_{\Xbb}(f_{.},g))>0$, hence $\Lcal_{\Xbb}(f_{.},g)>0$ with nonzero probability. This ends the proof of the result.

\section{Proof of Theorem \ref{thm:evans}}
\label{appendix:proof_evans}

Because $\Ycal$ is separable, the sequence of empirical measures converges weakly to the true measure of $Y$ almost surely \cite{varadarajan1958convergence}. Hence, by Lemma 2.1 of \cite{evans2020strong}, almost surely we have for all $y\in \Ycal$, $\frac{1}{n}\sum_{i=1}^n d^p(y,Y_i)\to \Ebb d^p(y,Y)$. In the rest of this proof, we suppose that this event is met. We now denote by $R^*:= \min_{y\in \Ycal}\Ebb d^p(y,Y)$ the minimal risk and $y_0\in \Ycal$ such that $\Ebb d^p(y_0,Y)<\infty$. We first note that the sequence $(\hat y_n)_n$ is bounded almost surely since $d^p(\hat y_n,y_0)\leq \frac{2^p}{n}\sum_{i=1}^n d^p(y_0,Y_i)\to 2^p\Ebb d^p(y_0,Y)$. Therefore, almost surely the sequence lies in a compact. We will suppose that this condition is also met in the rest of the proof. Now suppose by contradiction that the convergence does not hold. Let $\epsilon>0$ and consider a subsequence $\phi$ such that $|\frac{1}{\phi(k)}\sum_{i=1}^{\phi(k)} d^p(\hat y_{\phi(k)},Y_i)-R^*|\geq \epsilon$. Because the sequence $\hat y_n$ is confined to a compact closed ball, there exists a subsequence $\psi$ and $y^*\in \Ycal$ such that $\hat y_{\phi(\psi(l))}\to y^*$. For simplicity, we omit the subsequences and simply write $\hat y_l\to y^*$. Fix $\epsilon>0$. Using the constant $c_\epsilon$ dependent on $p$ from Lemma 2.3 of \cite{evans2020strong} such that for any $a,b\geq 0$, $(a+b)^p\leq (1+\epsilon)a^p + c_\epsilon b^p$, we can write for any $y\in \Ycal$,
\begin{align*}
    \Ebb d^p(y^*,Y) =\lim \frac{1}{l}\sum_{i=1}^{l} d^p(y^*,Y_i) &\leq \liminf c_\epsilon d^p(y^*,\hat y_l) + \frac{1+\epsilon}{l}\sum_{i=1}^{l} d^p(\hat y_l,Y_i)\\
    &\leq \liminf c_\epsilon d^p(y^*,\hat y_l) + \frac{1+\epsilon}{l}\sum_{i=1}^{l} d^p(y,Y_i)\\
    &=(1+\epsilon)\Ebb d^p(y,Y).
\end{align*}
Since this holds for any $\epsilon>0$ and $y\in \Ycal$, this shows that $\Ebb d^p(y^*,Y)=R^*$. We now observe that
\begin{align*}
    \limsup\frac{1}{l}\sum_{i=1}^{l} d^p(\hat y_l,Y_i)& \leq \limsup c_\epsilon d^p(\hat y_l,y^*) + \frac{1+\epsilon}{l}\sum_{i=1}^{l} d^p(y^*,Y_i) = (1+\epsilon)\Ebb d^p(y^*,Y),\\
    \liminf\frac{1}{l}\sum_{i=1}^{l} d^p(\hat y_l,Y_i)&\geq \limsup  \frac{1}{(1+\epsilon)l}\sum_{i=1}^{l} d^p(y^*,Y_i) - \frac{c_\epsilon}{1+\epsilon}d(\hat y_l,y^*)  = \frac{\Ebb d^p(y^*,Y)}{1+\epsilon}.
\end{align*}
Therefore, $\frac{1}{l}\sum_{i=1}^{l} d^p(\hat y_l,Y_i)\to R^*$ which contradicts the hypothesis and ends the proof of the theorem.

\section{Proof of Theorem \ref{thm:bayes_optimistic}}\label{appendix:proof_bayes}
We fix a process $\Xbb$ that admits a strong universal Bayes consistent learning rule and define $S=\{x\in \Xcal,\; \{x\}\cap \Xbb\neq \emptyset\}$ the random support of $\Xbb$. This set is almost surely finite because $\Xbb\in\suol$. We fix a measurable function $\bar f$ and we write,
\begin{align*}
    \mathcal{L}_{\Xbb}(f_{\cdot},p_{Y|X};T,\fb) &= \limsup_{T\rightarrow\infty}\frac{1}{T}\sum_{t=1}^T\sum_{x\in S, X_t = x}\left(\ell(f_t(\Xbb_{<t},\Ybb_{<t},X_t),Y_t)-\ell(\bar{f}(X_t),Y_t)\right),\\
    &\leq \sum_{x\in S}\limsup_{T\rightarrow\infty}\frac{1}{T}\sum_{t=1}^T \1_{X_t=x}\left(\ell(f_t(\Xbb_{<t},\Ybb_{<t},X_t),Y_t)-\ell(\bar{f}(X_t),Y_t)\right).
\end{align*}
Now fix $x\in \Xcal$, And denote $\Lcal_x(T):=\frac{1}{T}\sum_{t=1}^T \1_{X_t=x}\left(\ell(f_t(\Xbb_{<t},\Ybb_{<t},X_t),Y_t)-\ell(\bar{f}(X_t),Y_t)\right).$ If $\sum_{t\geq 1}\1_{X_t=x}<\infty$, then we have directly $\limsup_T \Lcal_x(T)=0$. We now turn to the case $\sum_{t\geq 1}\1_{X_t=x}=\infty$. Note that the sequence $(Y_u)_{u,\;X_u=x}$ is an i.i.d. sequence of variables following the distribution $p_{Y|X=x}$. First suppose that there exists $y_0\in\Ycal$ such that $\Ebb_{Y|X=x}\ell(y_0,Y)<\infty$. Then, by Theorem \ref{thm:evans},
\begin{equation*}
     \frac{1}{\sum_{t= 1}^T\1_{X_t=x}}\sum_{t=1}^T\1_{X_t=x}\ell(f_t(\Xbb_{<t},\Ybb_{<t},X_t),Y_t)\to \min_{y\in \Ycal}\Ebb_{Y|X=x}\ell(y,Y).
\end{equation*}
Otherwise, we have observe directly
\begin{equation*}
    \frac{1}{\sum_{t= 1}^T\1_{X_t=x}}\sum_{t=1}^T\1_{X_t=x}\ell(f_t(\Xbb_{<t},\Ybb_{<t},X_t),Y_t)\leq \infty=\min_{y\in \Ycal}\Ebb_{Y|X=x}\ell(y,Y).
\end{equation*}
Further note that
\begin{equation*}
    \frac{1}{\sum_{t= 1}^T\1_{X_t=x}}\sum_{t=1}^T \1_{X_t=x}\ell(\bar{f}(X_t),Y_t)\to \Ebb_{Y|X=x}\ell(\bar f(x),Y)\geq \min_{y\in \Ycal}\Ebb_{Y|X=x}\ell(y,Y).
\end{equation*}
Therefore, we obtain in all cases
\begin{equation*}
    \limsup_{T\rightarrow\infty}\frac{1}{{\sum_{t= 1}^T\1_{X_t=x}}}\sum_{t=1}^T \1_{X_t=x}\left(\ell(f_t(\Xbb_{<t},\Ybb_{<t},X_t),Y_t)-\ell(\bar{f}(X_t),Y_t)\right) \leq 0.
\end{equation*}
Because we have $\sum_{t= 1}^T\1_{X_t=x}\leq T$, this yields $\limsup_{T}\Lcal_x(T)\leq 0$.
Now recall that because $\Xbb\in\suol$, the set $S$ is finite almost surely, hence $\mathcal{L}_{\Xbb}(f_{\cdot},p_{Y|X};T,\fb)\leq 0\; (a.s.).$

\end{document}